\documentclass[11pt]{article}
\usepackage[margin=1in]{geometry}
\usepackage{amsmath,amssymb,amsthm,mathtools}
\usepackage{booktabs}
\usepackage{multirow}
\usepackage{xcolor}
\usepackage{algorithm}
\usepackage{algpseudocode}
\usepackage[authoryear]{natbib}
\setcitestyle{authoryear,open={ ( },close={ ) },aysep={,}}
\usepackage{hyperref}

\hypersetup{colorlinks=true,linkcolor=blue,citecolor=blue,urlcolor=blue}

\newtheorem{proposition}{Proposition}

\title{Technical Report: Activation Residual Hessian Quantization (ARHQ) for Low-Bit LLM Quantization}
\author{YiFeng Wang\thanks{wang.yifeng.q6@dc.tohoku.ac.jp}, Zhun Sun\thanks{sun.zhun.a5@tohoku.ac.jp}, Keisuke Sakaguchi \vspace{5pt} \\  \small{Graduate School of Information Sciences} \\ \small{Tohoku University}}
\date{}

\begin{document}
\maketitle

\vspace{-20pt}
\begin{abstract}
We present \emph{Activation Residual Hessian Quantization} (ARHQ), a post-training weight splitting method designed to mitigate error propagation in low-bit activation-weight quantization. By constructing an input-side residual Hessian from activation quantization residuals ($G_x$), ARHQ analytically identifies and isolates error-sensitive weight directions into a high-precision low-rank branch. This is achieved via a closed-form truncated SVD on the scaled weight matrix $WG_x^{1/2}$. Experimental results on Qwen3-4B-Thinking-2507 demonstrate that ARHQ significantly improves layer-wise SNR and preserves downstream reasoning performance on ZebraLogic even under aggressive quantization. The code is available at \url{https://github.com/BeautMoonQ/ARHQ}.
\end{abstract}

\section{Introduction}
Low-bit quantization of Large Language Models (LLMs) frequently encounters a fundamental mismatch between compression efficiency and inference fidelity. 
Conventional low-rank splitting methods typically approach this by addressing different sources of information loss: (1) minimizing the Frobenius energy to reconstruct the weight matrix $W$~\citep{wang2024svd}; (2) preserving heavy-hitter features by weighting the objective with activation energy $H_x = \frac{1}{N}X^\top X$~\citep{yuan2023asvd}; or (3) absorbing magnitude outliers into a high-precision branch to reduce the dynamic range and clipping errors of the quantized main branch~\citep{li2024svdquant}. However, we observe that under aggressive activation quantization, the dominant source of output error is often not the loss of original signal energy nor the presence of magnitude outliers alone, but rather how the linear map amplifies quantization residuals originating from the input activations. 

This observation motivates Activation Residual Hessian Quantization (ARHQ). Unlike previous works that aim to \emph{protect the signal} or \textit{absorb the outliers}, ARHQ is designed specifically to \textit{suppress the noise amplification}. 
By decomposing the output error of a quantized linear layer, we isolate the specific term $E_x W^\top$ responsible for propagating activation residuals $E_x$. ARHQ optimizes a low-rank split for this term, re-casting the task as a weighted low-rank approximation problem under a metric induced by the residual covariance $G_x = \frac{1}{N}E_x^\top E_x$. Intuitively, ARHQ identifies weight directions that most strongly amplify quantization noise and assigns them to a high-precision low-rank branch, while the remaining "residual-insensitive" components follow the standard quantized path. The optimization objective of ARHQ represents a significant departure from existing paradigms:  

\paragraph{Weight-only Reconstruction}~\citep{wang2024svd}: Minimizes pure weight approximation error by \begin{equation}
    \min_{L} \lVert W - L \rVert_F^2;
\end{equation}

\paragraph{Outlier Absorption}~\citep{li2024svdquant}: Extracts magnitude outliers into $L$ to minimize the quantization clipping error of the residual main branch by 
\begin{equation}
    \min_{L} \lVert (W - L) - Q_w(W - L) \rVert_F^2;
\end{equation}

\paragraph{Activation-aware Reconstruction}~\citep{yuan2023asvd}: Weights the reconstruction by the input activation covariance $H_x$ by
\begin{equation}\min_{L} \lVert X(W - L)^\top \rVert_F^2 \approx \min_{L} \lVert (W - L)H_x^{1/2} \rVert_F^2;
\end{equation}

\paragraph{Activation-Residual-aware Reconstruction}: Weights the reconstruction by the quantization residual covariance $G_x$ by 
\begin{equation}
\min_{L} \lVert E_x(W - L)^\top \rVert_F^2 = \min_{L} \lVert (W - L)G_x^{1/2} \rVert_F^2.
\end{equation}

We justify the term Residual Hessian by demonstrating that $G_x$ is precisely the shared input-side Hessian block of the residual propagation loss. This distinction is critical: since $G_x$ is derived directly from the observed quantization error $E_x$, ARHQ is inherently quantizer-adaptive. It adapts to the specificities of the deployed hardware codebooks, scaling rules, and calibration distributions, rather than assuming isotropic or quantizer-agnostic noise. 

\section{Activation Residual Hessian Quantization}
\subsection{Quantization Error Decomposition}

Consider a linear layer defined by
\begin{equation}
Y = X W^\top,
\end{equation}

where $X \in \mathbb{R}^{N \times D_{\mathrm{in}}}$ is the input activation matrix and $W \in \mathbb{R}^{D_{\mathrm{out}} \times D_{\mathrm{in}}}$ is the weight matrix.

Let $Q_x(\cdot)$ and $Q_w(\cdot)$ denote the activation and weight quantizers, respectively. We maintain a consistent notation for quantization residuals throughout this report:
\begin{equation}
    E_x = Q_x(X) - X, \qquad E_w = Q_w(W) - W
\end{equation}

The actual forward pass of the quantized layer is given by
\begin{equation}
    \hat{Y} = Q_x(X) Q_w(W)^\top = (X + E_x)(W + E_w)^\top
\end{equation}

Expanding this expression reveals the full output error
\begin{equation}
\hat{Y} - Y = E_x W^\top + X E_w^\top + E_x E_w^\top.
\label{eq:full_error}
\end{equation}

Equation~\eqref{eq:full_error} separates the total error into three distinct components: activation-error propagation ($E_x W^\top$), weight-error propagation ($X E_w^\top$), and their second-order interaction. ARHQ is specifically designed to tackle the first term, $E_x W^\top$, as it directly captures how the layer amplifies activation quantization noise, which is considered a dominant source of degradation in low-bit regimes, and conveniently admits a tractable closed-form solution.

\subsection{Isolating the Residual Propagation via a Low-Rank Branch}
To mitigate this error amplification, ARHQ splits the weight matrix into two fundamental components:
\begin{equation}
    W = W_{\mathrm{res}} + L, \qquad \mathrm{rank}(L) \le r
\end{equation}

Here, $W_{\mathrm{res}}$ remains on the primary quantized branch, while $L$ acts as a high-precision, low-rank side branch. In practical deployments, to satisfy the rank constraint and achieve actual computational and memory savings, this low-rank matrix is explicitly parameterized as the product of two factor matrices as $L = BA^\top$, where $A \in \mathbb{R}^{D_{\mathrm{in}} \times r}$ acts as a down-projection and $B \in \mathbb{R}^{D_{\mathrm{out}} \times r}$ acts as an up-projection (the exact analytical forms of $B$ and $A$ will be derived in Section 2.5).

During actual inference, the forward pass of the layer is structurally approximated by these two parallel branches as 
\begin{equation}
    \hat{Y} = Q_x(X) Q_w(W_{\mathrm{res}})^\top + X L^\top
\end{equation}

To derive a closed-form split for $L$, we must isolate the activation-residual propagation. We do this by temporarily assuming the main branch $W_{\mathrm{res}}$ is kept in full precision, yielding a surrogate inference model denoted as

\begin{equation}
    \hat{Y}_{\mathrm{approx}} = Q_x(X) W_{\mathrm{res}}^\top + X L^\top
\end{equation}

Substituting $W_{\mathrm{res}} = W - L$ into this surrogate model, we evaluate the gap between the exact floating-point output and the surrogate approximation

\begin{equation}
    Y - \hat{Y}_{\mathrm{approx}} = XW^\top - Q_x(X)(W-L)^\top - XL^\top = (X - Q_x(X))(W-L)^\top.
\end{equation}

Since the activation residual is defined as $E_x = Q_x(X) - X$, the error magnitude becomes exactly $E_x(W-L)^\top$. Minimizing the Frobenius energy of this isolated error leads directly to the ARHQ primal objective as
\begin{equation}
    \min_{\mathrm{rank}(L) \le r} \lVert E_x (W-L)^\top \rVert_F^2
\end{equation}

\subsection{The Activation Residual Hessian}

To solve this objective analytically, we define the residual covariance matrix in the following form:

\begin{equation}
G_x = \frac{1}{N} E_x^\top E_x.
\label{eq:gx_def}
\end{equation}

Using the cyclic property of the trace, we can rewrite the primal objective as
\begin{align}
\lVert E_x (W-L)^\top \rVert_F^2
&= \mathrm{Tr}\!\left((W-L)E_x^\top E_x(W-L)^\top\right) \\
&= N\,\mathrm{Tr}\!\left((W-L)G_x(W-L)^\top\right) \\
&= N\,\lVert (W-L)G_x^{1/2} \rVert_F^2.
\end{align}

Dropping the constant factor $N$, the optimization problem simplifies to a weighted low-rank approximation that solves
\begin{equation}
\min_{\mathrm{rank}(L) \le r} \lVert (W-L)G_x^{1/2} \rVert_F^2.
\label{eq:weighted_objective}
\end{equation}

It is important to explicitly distinguish the residual covariance $G_x$ from the standard activation covariance $H_x = \frac{1}{N} X^\top X$. In standard post-training quantization literature~\citep{frantar2022gptq}, weight importance is typically governed by the activation Hessian $H_x$, which arises naturally as the second-order derivative of the standard signal reconstruction loss. Methods relying on $H_x$ inherently assume that preserving the largest original input features is the optimal strategy. In contrast, ARHQ relies on $G_x$ to directly capture the structure of the quantization noise itself. This ensures that the low-rank split prioritizes fixing the errors introduced by activation quantization, rather than generic signal reconstruction.

Because our primary focus shifts from full signal reconstruction to mitigating the isolated residual propagation error, the optimization landscape is correspondingly governed by $G_x$. We therefore formally term $G_x$ the \textit{activation residual Hessian}. To see this mathematically, consider the residual propagation loss derived from our isolated objective: $\mathcal{L}(W_{\mathrm{res}}) = \frac{1}{N}\lVert E_x W_{\mathrm{res}}^\top \rVert_F^2 = \mathrm{Tr}(W_{\mathrm{res}}G_xW_{\mathrm{res}}^\top)$. Its gradient and Hessian with respect to $W_{\mathrm{res}}$ can be exactly written as:
\begin{equation}
\nabla_{W_{\mathrm{res}}} \mathcal{L} = 2W_{\mathrm{res}}G_x,
\qquad
\nabla^2_{W_{\mathrm{res}}} \mathcal{L} = 2I_{D_{\mathrm{out}}} \otimes G_x.
\end{equation}
Thus, $G_x$ plays the exact same mathematical and functional role in our residual-aware framework as the standard Hessian $H_x$ does in traditional methods like GPTQ~\citep{frantar2022gptq}.

\subsection{Closed-form solution}

\begin{proposition}[ARHQ closed-form decomposition]
Let $W \in \mathbb{R}^{D_{\mathrm{out}} \times D_{\mathrm{in}}}$ and assume the activation residual covariance $G_x \succ 0$. Define the scaled weight matrix $M = WG_x^{1/2}$. If $M_r$ is the optimal rank-$r$ approximation of $M$ under the Frobenius norm, then the exact solution to the weighted objective
\begin{equation}
 \min_{\mathrm{rank}(L)\le r} \lVert (W-L)G_x^{1/2} \rVert_F^2
\end{equation}
is strictly given by
\begin{equation}
    L^\star = M_r G_x^{-1/2}
\end{equation}
\end{proposition}

\begin{proof}
We introduce the change of variables $\widetilde{L} = L G_x^{1/2}$. Because $G_x \succ 0$, the matrix $G_x^{1/2}$ is invertible, which guarantees that right-multiplication preserves the matrix rank, i.e., $\mathrm{rank}(\widetilde{L}) = \mathrm{rank}(L)$. The weighted objective can thus be rewritten as a standard unweighted approximation problem:
\begin{equation}
    \min_{\mathrm{rank}(\widetilde{L})\le r} \lVert M - \widetilde{L} \rVert_F^2
\end{equation}

By the Eckart-Young-Mirsky theorem, the optimal solution for $\widetilde{L}$ is the truncated SVD of $M$, denoted as $M_r$. Mapping this solution back to the original parameter space yields $L^\star = M_r G_x^{-1/2}$, concluding the proof.
\end{proof}

\subsection{Factorization for Deployment}
While $L^\star$ represents the mathematical optimum, deploying it directly as a dense matrix offers no computational savings. We must factorize it into two smaller matrices to construct the low-rank side branch $L = BA^\top$.

Let the truncated SVD of the scaled weight matrix be explicitly written as $M_r = U_r \Sigma_r V_r^\top$. Substituting this into our optimal solution gives
\begin{equation}
    L^\star = (U_r \Sigma_r V_r^\top) G_x^{-1/2}
\end{equation}
To match the parameter shape of a dual-layer linear bottleneck, we group the terms to define the down-projection matrix $A \in \mathbb{R}^{D_{\mathrm{in}} \times r}$ and the up-projection matrix $B \in \mathbb{R}^{D_{\mathrm{out}} \times r}$ as follows
\begin{equation}
    B = U_r \Sigma_r, \qquad A = G_x^{-1/2} V_r
\end{equation}
This yields the final operational form $L^\star = BA^\top$. During inference, the input $X$ is sequentially multiplied by $A$ and $B^\top$, perfectly mirroring the standard LoRA architecture but achieved entirely post-training without backpropagation.

\subsection{Numerical regularization}

In practical deployments, $G_x$ is frequently singular or heavily ill-conditioned. This typically occurs when the number of calibration samples is smaller than the input dimension $D_{\mathrm{in}}$, when quantization residuals are near zero for specific channels, or when prior smoothing operations artificially lower the effective rank.To handle this instability, we compute the eigendecomposition of the residual covariance, $G_x = U\,\mathrm{diag}(\lambda)\,U^\top$, and apply a floor $\varepsilon$ to the eigenvalues to create a regularized metric as 

\begin{equation}
G_{x,\varepsilon} = U\,\mathrm{diag}(\max(\lambda_i,\varepsilon))\,U^\top.
\end{equation}

The safely deployed objective therefore becomes
\begin{equation}
\min_{\mathrm{rank}(L)\le r} \lVert (W-L)G_{x,\varepsilon}^{1/2} \rVert_F^2.
\end{equation}
This regularization stabilizes the inversion $G_{x,\varepsilon}^{-1/2}$ and prevents the unconstrained growth of $L$ along directions with near-zero residual energy, ensuring the problem remains identifiable and numerically robust.

\begin{algorithm}[H]
\caption{Activation Residual Hessian Quantization (ARHQ) with Smoothing}
\label{alg:arhq_smoothing}
\begin{algorithmic}[1]
\Require Original weight matrix $W \in \mathbb{R}^{D_{out} \times D_{in}}$, Calibration activations $X_{calib} \in \mathbb{R}^{N_{calib} \times D_{in}}$, Positive diagonal smoothing matrix $S = \text{diag}(s)$, Activation quantizer $Q_x(\cdot)$, Target rank $r$, Regularization floor $\epsilon$
\Ensure Main branch residual weight $W_{res,s}$, Down-projection matrix $A_s$, Up-projection matrix $B$

\Statex \textbf{\textit{Stage 1: Smoothing Preprocessing}}
\State $X_s \gets X_{calib}S^{-1}$ \Comment{Smooth calibration activations}
\State $W_s \gets WS$ \Comment{Smooth weight matrix}

\Statex \textbf{\textit{Stage 2: Residual Covariance Computation}}
\State $X_q \gets Q_x(X_s)$ \Comment{Fake-quantize in the smoothed space}
\State $E_s \gets X_q - X_s$ \Comment{Compute smoothed activation residuals}
\State $G_s \gets \frac{1}{N_{calib}} E_s^\top E_s$ \Comment{Construct residual covariance (Hessian)}

\Statex \textbf{\textit{Stage 3: Numerical Regularization}}
\State Compute eigendecomposition $G_s = U \text{diag}(\lambda) U^\top$
\State $G_{s,\epsilon} \gets U \text{diag}(\max(\lambda_i, \epsilon)) U^\top$ \Comment{Apply eigenvalue floor}
\State Construct $G_{s,\epsilon}^{1/2}$ and $G_{s,\epsilon}^{-1/2}$

\Statex \textbf{\textit{Stage 4: Scaled Truncated SVD}}
\State $M_s \gets W_s G_{s,\epsilon}^{1/2}$ \Comment{Scale weights by residual metric}
\State $U_r, \Sigma_r, V_r \gets \text{TruncatedSVD}(M_s, r)$ \Comment{Extract top-$r$ components}

\Statex \textbf{\textit{Stage 5: Factorization for Deployment}}
\State $B \gets U_r \Sigma_r$ \Comment{Construct up-projection factor}
\State $A_s \gets G_{s,\epsilon}^{-1/2} V_r$ \Comment{Construct down-projection factor}
\State $L_s \gets B A_s^\top$ \Comment{Form high-precision low-rank branch}
\State $W_{res,s} \gets W_s - L_s$ \Comment{Compute residual weight for main branch}

\State \Return $W_{res,s}, A_s, B$
\end{algorithmic}
\end{algorithm}

\section{Practical Pipeline}
ARHQ reduces the propagation of activation residuals, but it does not directly shrink the residual magnitude. Therefore, it is highly complementary to activation smoothing. Smoothing reduces or redistributes the activation error, while ARHQ subsequently adapts the low-rank split to the residual structure that remains after the smoothing operation. We integrate these two mechanisms into a unified practical pipeline.  

Let $S = \text{diag}(s)$ be a positive diagonal smoothing matrix. We first apply an equivalent transform to the calibration activations and the weights
\begin{equation}
    X_s = X S^{-1}, \qquad W_s = W S.
\end{equation}

This transformation preserves the full-precision output, given that 
\begin{equation}
    X_s W_s^\top = X S^{-1} (W S)^\top = X W^\top.
\end{equation}
Operating entirely within this smoothed space, the unified ARHQ decomposition is computed offline using calibration data, as detailed in Algorithm 1.  

At runtime, inference seamlessly integrates the smoothed parameters. Assuming the pre-quantized, smoothed activation $X S^{-1}$ is available at the linear layer boundary before online activation quantization, the forward pass is structurally approximated by
\begin{equation}
    \hat{Y} = Q_x(X S^{-1}) Q_w(W_{res,s})^\top + (X S^{-1} A_s) B^\top
\end{equation}

In this configuration, the main branch computes the quantized matrix multiplication 
\begin{equation}
    Q_x(X S^{-1}) Q_w(W_{res,s})^\top, 
\end{equation}
while the side branch computes the higher-precision projection $(X S^{-1} A_s) B^\top$. For a single layer, this dual-branch architecture introduces an additional arithmetic cost of $\mathcal{O}(N r (D_{in} + D_{out}))$. The additional parameter count is $r(D_{in} + D_{out})$. For square layers where $D_{in} = D_{out} = D$, this parameter overhead ratio simplifies to $2r/D$.  

\section{Experimental Evidence}

\subsection{Setup summary}

We summarize the current experimental scope used to validate the method:
\begin{itemize}
    \item \textbf{Model.} Qwen3-4B-Thinking-2507.
    \item \textbf{Modules.} Attention projections across 36 layers.
    \item \textbf{Rank.} $r=128$.
    \item \textbf{Quantization focus.} Aggressive low-bit activation-weight quantization, including NVFP4-style settings in the reported SNR study.
    \item \textbf{Baselines.} A standard SVD low-rank split with and without smoothing.
\end{itemize}

Several implementation details remain to be fully documented in a final reproducibility package, including exact block sizes, scale computation rules, accumulation dtypes, and whether low-rank factors are stored in bf16 or fp16. We state these omissions explicitly because they matter for a complete arXiv release.

\subsection{SNR definition}

For a floating-point reference output $Y$ and a quantized approximation $\hat{Y}$, we use
\begin{equation}
\mathrm{SNR}(Y,\hat{Y}) = 10\log_{10}\frac{\lVert Y \rVert_F^2}{\lVert Y-\hat{Y} \rVert_F^2}.
\end{equation}
Higher values are better.

\subsection{Layer-wise attention projection SNR}

Table~\ref{tab:snr} reports the current SNR measurements on Qwen3-4B-Thinking-2507 attention projections. ARHQ raw corresponds to residual-aware splitting without smoothing; ARHQ smoothing adds smoothing before decomposition.

\begin{table}[t]
\centering
\caption{Layer-wise SNR (dB) on 36 attention layers, rank $r=128$.}
\label{tab:snr}
\begin{tabular}{lcccc}
\toprule
Scope & ARHQ Raw & ARHQ Smooth & SVD Raw & SVD Smooth \\
\midrule
q/k/v/o average & 24.4229 & 24.7269 & 22.6314 & 24.2742 \\
q\_proj & 28.3389 & 28.6415 & 25.7253 & 28.1420 \\
k\_proj & 28.3682 & 28.6888 & 25.8594 & 27.9529 \\
v\_proj & 21.4883 & 21.7858 & 19.6953 & 21.4014 \\
o\_proj & 19.4964 & 19.7917 & 19.2456 & 19.6006 \\
\bottomrule
\end{tabular}
\end{table}

ARHQ outperforms the unweighted SVD split without smoothing by a clear margin. On the q/k/v/o average, ARHQ raw reaches $24.4229$ dB versus $22.6314$ dB for SVD raw, a gain of about $1.79$ dB. With smoothing, ARHQ remains best at $24.7269$ dB, also improving over SVD smoothing at $24.2742$ dB. This pattern is consistent with the claim that residual-aware weighting chooses more useful low-rank directions than weight-only reconstruction.

\subsection{Improvement over quantized baseline}

Table~\ref{tab:snr_gain} reports the improvement over the underlying quantized baseline.

\begin{table}[t]
\centering
\caption{Improvement over baseline SNR (dB).}
\label{tab:snr_gain}
\begin{tabular}{lcccc}
\toprule
Scope & ARHQ Raw & ARHQ Smooth & SVD Raw & SVD Smooth \\
\midrule
q/k/v/o average & +4.0828 & +3.7906 & +2.5627 & +3.3378 \\
q\_proj & +5.1668 & +4.7555 & +3.2661 & +4.2560 \\
k\_proj & +6.5745 & +6.1158 & +4.2142 & +5.3799 \\
v\_proj & +2.9964 & +2.6233 & +1.5308 & +2.2389 \\
o\_proj & +1.5936 & +1.6675 & +1.2397 & +1.4764 \\
\bottomrule
\end{tabular}
\end{table}

These numbers suggest two practical conclusions. First, residual-aware weighting matters even before smoothing. Second, smoothing and ARHQ appear complementary rather than redundant.

\subsection{Small-scale ZebraLogic evaluation}

We also ran a preliminary generation-based evaluation on a 140-problem ZebraLogic subset using Qwen3-4B-Thinking-2507 with seed $0$, temperature $0.6$, top-$p=0.95$, and top-$k=20$.

\begin{table}[t]
\centering
\caption{Preliminary ZebraLogic results on 140 puzzles.}
\label{tab:zebra}
\begin{tabular}{lccc}
\toprule
Method & Puzzle Acc & Cell Acc & Cell Recall \\
\midrule
bf16 & 130/140 = 92.86\% & 2671/2732 = 97.77\% & 2671/2837 = 94.15\% \\
nvfp4 & 105/140 = 75.00\% & 2279/2625 = 86.82\% & 2279/2783 = 81.89\% \\
awq4bit & 125/140 = 89.29\% & 2589/2771 = 93.43\% & 2589/2873 = 90.11\% \\
svdquant\_r128 & 127/140 = 90.71\% & 2643/2766 = 95.55\% & 2643/2873 = 91.99\% \\
arhq\_r128\_all & 130/140 = 92.86\% & 2660/2779 = 95.72\% & 2660/2873 = 92.59\% \\
\bottomrule
\end{tabular}
\end{table}

On this small subset, ARHQ achieves the highest observed accuracy among quantized variants and matches the bf16 run on puzzle accuracy. However, the sample size is too small for strong conclusions. For $130/140$, the binomial standard error is about $2.18\%$, so a 95\% interval is roughly $\pm 4.3\%$. The difference between $127/140$ and $130/140$ is therefore within plausible statistical noise. We view Table~\ref{tab:zebra} as encouraging evidence rather than definitive proof.

\section{Discussion}

\subsection{Why residual Hessian instead of activation Hessian?}

Using $H_x = X^\top X/N$ would prioritize directions with large activation energy, but low-bit quantization error is not generally isotropic and need not align with those directions. ARHQ instead uses $G_x = E_x^\top E_x/N$, which is directly measured from the deployed quantizer and calibration data. This makes the method sensitive to block-wise quantization, non-uniform floating-point codebooks, scale selection rules, and any layer-specific distortion pattern.

\subsection{Quantizer adaptivity}

ARHQ is adaptive not only to the quantizer family but to the triple of \emph{quantizer, calibration distribution, and layer input distribution}. More precisely,
\begin{equation}
G_x = \mathbb{E}_{x\sim\mathcal{D}_{\mathrm{calib}}}\big[(Q_x(x)-x)(Q_x(x)-x)^\top\big].
\end{equation}
This is a strength when calibration matches deployment, but also a limitation: a mismatched calibration set can misestimate the true residual structure and reduce the benefit of the split.

\subsection{Interaction with weight quantization}

The exact ARHQ objective solves the activation-residual subproblem, not the full joint objective in Equation~\eqref{eq:full_error}. After the split, $W_{\mathrm{res}}$ is still quantized, which introduces additional terms involving $E_w$. In some cases, the inverse square-root mapping can enlarge the dynamic range of the low-rank factors or the residual weight, potentially making $W_{\mathrm{res}}$ harder to quantize. We partially address this with eigenvalue clamping and consider quantization-aware post-split refinement an important direction for future work.

\section{Limitations and Future Work}

The current formulation is deliberately narrow so that it admits a closed-form solution. Several important directions remain open.

\textbf{Weight-quantization-aware ARHQ.} A more complete objective would combine activation residual propagation with residual-weight quantization loss, for example
\begin{equation}
\min_L \lVert E_x(W-L)^\top \rVert_F^2 + \lambda\,\lVert X((W-L)-Q_w(W-L))^\top \rVert_F^2,
\end{equation}
but this no longer has a simple closed form because the quantizer depends on $L$.

\textbf{Joint smoothing and splitting.} In this report, smoothing is a preprocessing step. Jointly optimizing the smoothing transform and the ARHQ split may yield better trade-offs.

\textbf{Rank allocation.} All current results use a fixed rank $r=128$. A natural extension is to allocate a global rank budget according to the singular value gains of $WG_x^{1/2}$ at each layer.

\textbf{Lower-cost variants.} Full covariance ARHQ requires eigendecomposition of a $D_{\mathrm{in}}\times D_{\mathrm{in}}$ matrix. Diagonal or block-diagonal approximations may offer a better quality-cost trade-off.

\textbf{Stronger evaluation.} A complete arXiv version should add rank ablations, comparisons against $H_x$-weighted and identity-weighted objectives, calibration-size studies, layer-wise spectra, perplexity, broader reasoning benchmarks, and end-to-end latency and memory measurements.



\bibliographystyle{plainnat}
\bibliography{tech_report}

\end{document}